\documentclass[11pt]{article}

\usepackage{amsmath, amssymb, amsthm}
\usepackage{geometry}
\usepackage{hyperref}
\usepackage{graphicx}
\usepackage{booktabs}
\usepackage{setspace}
\usepackage{microtype}
\usepackage{xcolor}
\usepackage[numbers]{natbib}

\newtheorem{theorem}{Theorem}
\newtheorem{proposition}{Proposition}

\newtheorem{definition}{Definition}
\newtheorem{remark}{Remark}

\title{\textbf{Disease Is a Spectral Perturbation}}

\author{
  John D.\ Mayfield, MD, PhD, MSc\textsuperscript{1,2} \and
  Matthew S.\ Rosen, PhD\textsuperscript{1,2}
}

\date{
  \textsuperscript{1}Athinoula A.\ Martinos Center for Biomedical Imaging,
  Massachusetts General Hospital\\
  \textsuperscript{2}Department of Radiology, Harvard Medical School,
  Boston, MA 02129\\[6pt]
  \small Correspondence: \texttt{jdmayfield@mgh.harvard.edu}
}

\begin{document}

\maketitle

\begin{abstract}

We propose a novel method of understanding disease transformation from a healthy baseline with biomarker-level explainability. By modeling the biomarker covariance matrices of healthy controls and disease states, the perturbation can be individually characterized to accomplish mechanistic explanations of disease trajectories, both at a molecular level and for individual patients. Given a cohort of $n$ patients each measured on $p$ biomarkers, we define the \emph{biomarker "Hamiltonian"} $H = X^\top X / n \in \mathbb{R}^{p \times p}$, where $X \in \mathbb{R}^{n \times p}$ is the covariant biomarker matrix. The eigenvectors of $H$ define a set of normal modes of biomarker coordination, and the eigenvalues quantify the energy carried by each mode. In the healthy state, the reference Hamiltonian $H_0$ governs this structure where disease perturbs $H_0$ by an additive operator $\Delta H$, thus shifting eigenvalues and rotating eigenvectors in proportion to the severity of pathological disruption. We formalize this framework, derive the spectral change given a disease perturbation, and demonstrate that the projection of a newly diagnosed patient's cumulative biomarker covariance structure onto disease-discriminant eigenmodes constitutes an optimal prognostic statistic for greater precision in disease prognosis. This work serves as a veritable white paper with application across a panoply of disease frameworks from cancer to neurodegenerative disorders.
\end{abstract}


\vspace{12pt}

\section{Introduction}

The fundamental challenge in quantitative medicine is not the acquisition of biomarker data but its interpretation, especially given the growing number of readily available diagnostic tests coming into practice. High-dimensional multiomics profiles now routinely characterize patients across genomic, transcriptomic, proteomic, and metabolomic layers \citep{Hasin2017, Babu2023}. Longitudinal cohorts track how these profiles evolve across disease states, and trajectory modeling frameworks attempt to extract the prognostic signal embedded in this evolution \citep{Nagin2018, Hjaltelin2023}. Yet, despite the sophistication of individual methods such as joint latent class models \citep{ProustLima2022}, multistate models \citep{Jiang2025circ}, temporal biomarker detection \citep{Metwally2021}, and multiomic integration \citep{Maghsoudi2022}, the field lacks a unifying mathematical formalism from which these approaches can be derived as special cases and against which their assumptions can be evaluated.

We argue that this object exists as the covariance matrix of biomarker measurements, $H = X^\top X / n$. This matrix, which we term the \emph{biomarker "Hamiltonian"} by analogy with the energy operator in quantum mechanics, encodes the full second-order statistical structure of a patient cohort. Its eigenvectors define the principal axes of biomarker covariation, and its eigenvalues measure the variance, rather the energy, concentrated along each axis. A disease state is not merely a shift in mean biomarker levels. It is a reorganization of covariance structure where certain modes of coordinated variation are amplified, others suppressed, and the dominant eigenmodes rotate away from those of the healthy reference. In this precise sense, \emph{disease is a spectral perturbation}.

This framing is not merely geometric as it carries operational consequences. If the healthy covariance structure $H_0$ is known, then the perturbation $\Delta H = H_d - H_0$ induced by disease has a well-characterized spectral signature. Matrix perturbation theory, specifically the Weyl inequalities and the Davis-Kahan theorem, provides discrete bounds on how much the eigenvalues and eigenvectors of $H_0$ can shift under $\Delta H$ as a function of the perturbation magnitude. These bounds translate directly into statements about the differentiation of disease states, the stability of prognostic scores across cohorts, and the conditions under which a spectral signature from one population generalizes to another.

Current multiomics trajectory methods are powerful, yet remain fragmented. Group-based trajectory modeling identifies latent subgroups with shared biomarker dynamics using finite mixture models \citep{Nagin2018, NaginJones2024} but does not characterize the coordination structure within or between groups. Joint longitudinal models estimate individual trajectories with high efficiency \citep{Li2021, Singh2025} but operate on mean trajectories rather than on the specific eigenmodes of variation patterns between biomarkers. Machine learning integration methods such as JIVE (Joint and Individual Variation Explained), MOFA (Multi-Omics Factor Analysis), intNMF (Integrative Non-negative Matrix Factorization) \citep{Maghsoudi2022, Morabito2025} decompose multiomics data into shared and modality-specific factors, however, do not ground these factors in a reference covariance structure relative to which disease perturbation is defined. Proteomic and multiomic prediction models achieve strong performance \citep{Carrasco2024, Ouwerkerk2023} but derive their prognostic scores empirically rather than from the overall spectral geometry of the covariance operator.

Each of these methods operates on a projection or special case of the covariance eigenstructure without naming the object onto which it is projecting nor proving the optimality of that projection. This paper attempts to provide the unifying formalism. 

The paper is organized as follows. Section~\ref{sec:formalism} defines the biomarker Hamiltonian and its spectral decomposition. Section~\ref{sec:perturbation} formalizes disease as a spectral perturbation and derives the principal theoretical results. Section~\ref{sec:score} derives the spectral prognostic score and proves its optimality. Section~\ref{sec:transfer} establishes conditions for eigenbasis transfer across cohorts. Section~\ref{sec:unification} shows that existing multivariate biomarker methods are special cases of this framework. Section~\ref{sec:discussion} situates the framework within the landscape of current multiomics disease modeling and identifies open problems.

\section{The Biomarker Hamiltonian}
\label{sec:formalism}

\subsection{Notation and Setup}

Let a patient cohort consist of $n$ individuals, each characterized by $p$ biomarkers. We represent the cohort as a data matrix $X \in \mathbb{R}^{n \times p}$, where $X_{ij}$ denotes the measurement of biomarker $j$ in patient $i$. Without loss of generality, assume columns of $X$ are mean-centered, so that $\sum_{i=1}^n X_{ij} = 0$ for all $j$.

\begin{definition}[Biomarker Hamiltonian]
The \emph{biomarker Hamiltonian} of cohort $X$ is the $p \times p$ matrix
\begin{equation}
H = \frac{1}{n} X^\top X.
\label{eq:hamiltonian}
\end{equation}
\end{definition}
where $H$ is the sample covariance matrix of the biomarker measurements, a symmetric positive semidefinite matrix by construction with $H = H^\top$ and $v^\top H v \geq 0$ for all $v \in \mathbb{R}^p$. Its rank is at most $\min(n, p)$.

The term \emph{Hamiltonian} is adopted deliberately. In quantum mechanics and classical statistical mechanics, the Hamiltonian is the energy operator of a system where its eigenvectors define the system's normal modes and its eigenvalues define the energy of each mode. The biomarker Hamiltonian plays an identical role at the cohort level. Its eigenvectors define the principal axes of biomarker covariance across patients, and its eigenvalues measure the variance (the energy) concentrated along each axis. This analogy is not decorative, but rather, motivates the thermodynamic extension developed in Section~\ref{sec:discussion}.

\subsection{Spectral Decomposition}

Since $H$ is real and symmetric, the spectral theorem guarantees a decomposition
\begin{equation}
H = Q \Lambda Q^\top,
\label{eq:spectral}
\end{equation}
where $Q = [q_1 \mid q_2 \mid \cdots \mid q_p] \in \mathbb{R}^{p \times p}$ is orthogonal, $Q^\top Q = I_p$, and $\Lambda = \mathrm{diag}(\lambda_1, \lambda_2, \ldots, \lambda_p)$ with eigenvalues ordered $\lambda_1 \geq \lambda_2 \geq \cdots \geq \lambda_p \geq 0$.

The columns $q_k$ are the \emph{eigenmodes} of the cohort representing orthonormal vectors in biomarker space along which the patient population is maximally, then successively, dispersed. The corresponding eigenvalues $\lambda_k$ measure the variance of patient projections onto $q_k$. Collectively, the eigenmodes form a complete orthonormal basis for $\mathbb{R}^p$, the \emph{spectral basis} of $H$.

\begin{remark}
The decomposition \eqref{eq:spectral} is identical to the principal component decomposition (PCA) of $X$. The distinction between PCA and the present framework is not computational but interpretive. PCA treats the eigenmodes as axes for dimensionality reduction by selecting the top-$k$ modes that explain a target fraction of total variance. The spectral biomarker framework treats the complete eigenbasis as a coordinate system for the space of possible disease perturbations from which dimensionality reduction is only a single application.
\end{remark}

\subsection{The Spiked Covariance Model and Disease Signal}

For large cohorts with $n, p \to \infty$ and $p/n \to \gamma \in (0, \infty)$, random matrix theory describes the limiting eigenvalue distribution of $H$ under the null hypothesis that $X$ has i.i.d.\ Gaussian entries. This limiting distribution is the Marchenko-Pastur law \citep{MarchenkoPastur1967}, with support $[(\sigma(1 - \sqrt{\gamma}))^2,\, (\sigma(1 + \sqrt{\gamma}))^2]$ for variance $\sigma^2$.

\textbf{Disease signal, by contrast, arises from structured covariance that escapes the bulk Marchenko-Pastur distribution.} In the \emph{spiked covariance model} \citep{Johnstone2001}, the population covariance has the form $\Sigma = I_p + \sum_{k=1}^r \theta_k u_k u_k^\top$, where $u_k$ are unit vectors and $\theta_k > 0$ are spike strengths. Eigenvalues corresponding to spikes separate from the bulk when $\theta_k > \sqrt{\gamma}$, the \emph{Baik–Ben Arous–Péché (BBP) transition} \citep{Baik2005}. Below this threshold, spikes are relatively invisible in the sample spectrum. A comparator of gained and lost covariance using synthetic data is demonstrated in Figure {\ref{fig:biologic_terrain}}.

This has a direct implication for disease characterization. A disease-induced perturbation of the covariance structure is detectable in the sample spectrum only when the perturbation magnitude exceeds $\sqrt{\gamma}$. This threshold defines the minimum detectable effect size for spectral biomarker methods and provides a principled sample size criterion such that $n$ must be large enough that $\sqrt{p/n}$ is small relative to the expected perturbation strength.

\begin{figure}[!t]
    \centering
    \includegraphics[width=\textwidth]{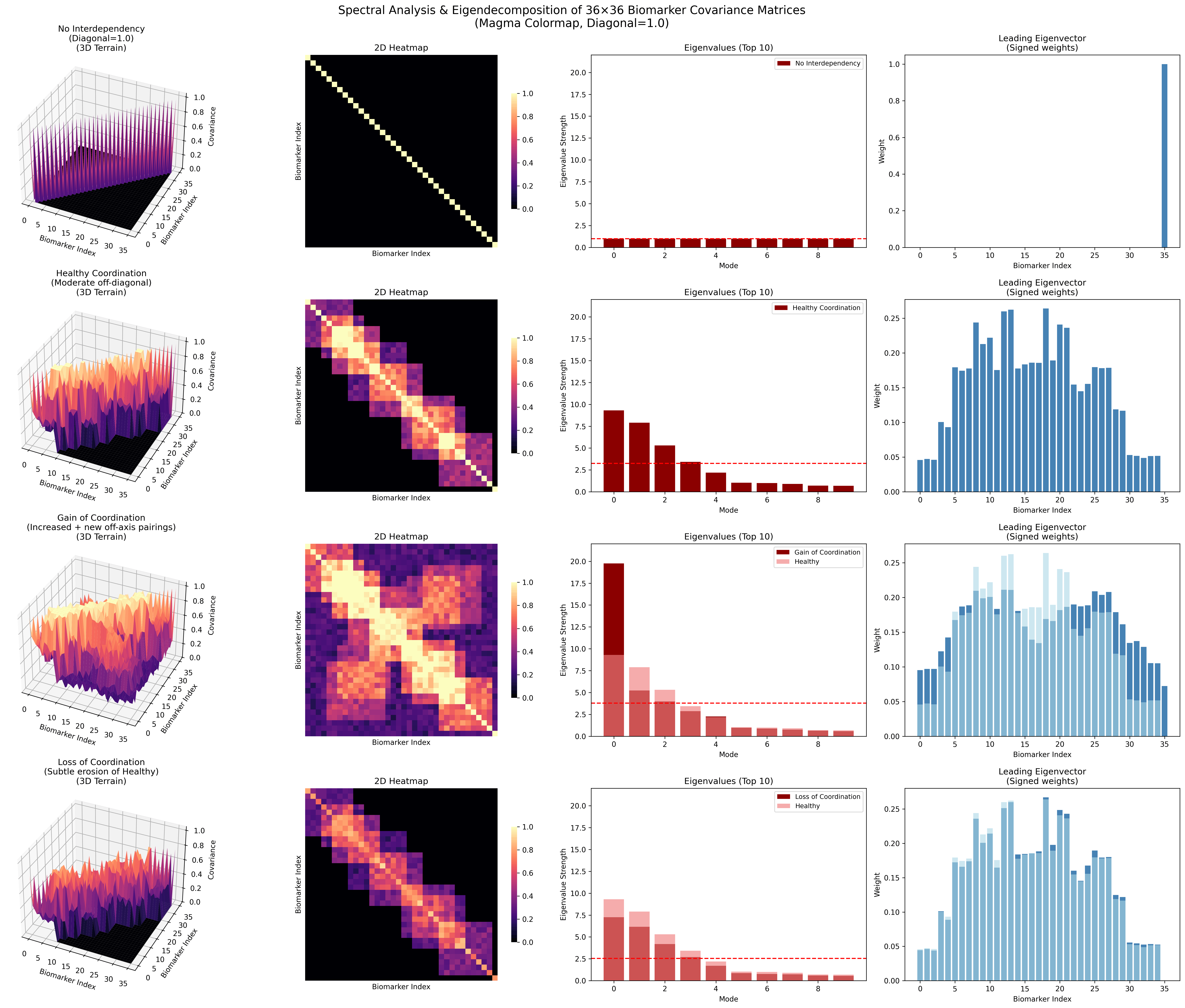}
    \caption{\textbf{Spectral signatures of four biologic terrain states across 36 biomarkers.}
    Each row represents a distinct covariance regime visualized across four panels.
    \textit{Left:} Three-dimensional terrain of the 36$\times$36 biomarker covariance matrix, where surface height encodes covariance magnitude.
    \textit{Center-left:} Two-dimensional heatmap of the same covariance matrix
    (Magma colormap; diagonal fixed at 1.0).
    \textit{Center-right:} Top 10 eigenvalues, with the dashed red line marking
    the Marchenko-Pastur threshold separating structured signal from
    sampling noise. Where applicable, the healthy reference spectrum
    is shown in light red for comparison.
    \textit{Right:} Signed weights of the leading eigenvector across all 36 biomarkers.
    \textbf{Row 1 (No Interdependency):} The identity covariance structure yields a flat eigenspectrum with a single dominant eigenvalue concentrated on one biomarker, indicating the absence of coordinated variation.
    \textbf{Row 2 (Healthy Coordination):} Moderate off-diagonal structure produces a distributed leading eigenvector with broad biomarker participation and several eigenvalues above the noise threshold.
    \textbf{Row 3 (Gain of Coordination):} Increased off-axis pairings amplify the dominant eigenvalue substantially beyond the healthy reference, reflecting
    pathological hypercorrelation along the leading mode.
    \textbf{Row 4 (Loss of Coordination):} Subtle erosion of healthy covariance structure attenuates eigenvalue strength toward the noise floor, consistent with a negative spectral disruption index ($\Phi < 0$).}
    \label{fig:biologic_terrain}
\end{figure}

\section{Disease as Spectral Perturbation}
\label{sec:perturbation}

\subsection{The Reference and Disease Hamiltonians}

Let $H_0 \in \mathbb{R}^{p \times p}$ denote the Hamiltonian of a reference healthy cohort and $H_d \in \mathbb{R}^{p \times p}$ the Hamiltonian of a disease cohort measured on the same $p$ biomarkers. We define the \emph{disease perturbation}
\begin{equation}
\Delta H = H_d - H_0.
\label{eq:perturbation}
\end{equation}

where $\Delta H$ is real symmetric as the difference of two real symmetric matrices. It need not be positive semidefinite; its negative eigenvalues correspond to modes of biomarker covariation that are suppressed by disease relative to the healthy reference.

Let $H_0 = Q_0 \Lambda_0 Q_0^\top$ and $H_d = Q_d \Lambda_d Q_d^\top$ be the spectral decompositions of the reference and disease Hamiltonians, respectively. The central objects of interest are the eigenvalue shifts $\delta\lambda_k = \lambda_k^d - \lambda_k^0$ and the eigenvector rotations, measured by the principal angles between the reference and disease eigenspaces.

\subsection{Eigenvalue Perturbation: Weyl's Inequality}

The following classical result controls eigenvalue shifts under symmetric perturbation.

\begin{theorem}[Weyl, 1912]
Let $A, B \in \mathbb{R}^{p \times p}$ be real symmetric matrices with eigenvalues $\alpha_1 \geq \cdots \geq \alpha_p$ and $\beta_1 \geq \cdots \geq \beta_p$, respectively. Then for all $k = 1, \ldots, p$,
\begin{equation}
|\alpha_k - \beta_k| \leq \|A - B\|_2,
\label{eq:weyl}
\end{equation}
where $\|\cdot\|_2$ denotes the spectral norm.
\end{theorem}

Applied to $A = H_d$ and $B = H_0$, Weyl's inequality gives
\begin{equation}
|\lambda_k^d - \lambda_k^0| \leq \|\Delta H\|_2
\label{eq:weyl_applied}
\end{equation}
for every mode $k$, where $\lambda_k^d$ and $\lambda_k^0$ are the $k$-th eigenvalues of $H_d$ and $H_0$, respectively. The spectral norm $\|\Delta H\|_2$ equals the magnitude of the largest eigenvalue of $\Delta H$. Thus the maximum possible shift in any single mode energy is bounded by the magnitude of the strongest mode of the disease perturbation itself.

\begin{proposition}[Spectral Disruption Index]
Define the \emph{spectral disruption index}
\begin{equation}
\Phi = \frac{\lambda_1^d - \lambda_1^0}{\lambda_1^0}
\label{eq:disruption}
\end{equation}

where \textbf{$\Phi$} is the fractional shift in the dominant eigenvalue. By \eqref{eq:weyl_applied}, $|\Phi| \leq \|\Delta H\|_2 / \lambda_1^0$. When $\Phi < 0$, disease reduces the dominant mode energy indicating disruption of the primary axis of healthy biomarker covariation. When $\Phi > 0$, disease amplifies it, indicating hypercorrelation along the dominant axis.
\end{proposition}

This formalism is key as \textbf{the sign and magnitude of $\Phi$ constitute a scalar characterization of disease severity}. Negative $\Phi$, or disrupted coordination of the comparative healthy mode, has been observed empirically across neurological disease cohorts \citep{Guo2025, Huang2025} which may represent immune system dysregulation or loss of homeostatic mechanisms that may be preventative of disease.

\subsection{Eigenvector Perturbation: The Davis-Kahan Theorem}

Eigenvalue shifts alone do not characterize the full spectral perturbation. Disease also rotates the eigenvectors, realigning the principal axes of biomarker covariation. The Davis-Kahan theorem controls this rotation.

\begin{theorem}[Davis-Kahan, 1970]
Let $A, B \in \mathbb{R}^{p \times p}$ be real symmetric, and let $V_A$, $V_B$ denote the eigenspaces corresponding to eigenvalues in an interval $[\alpha, \beta]$ for $A$ and $B$, respectively. If the eigenvalues of $B$ outside $[\alpha - \delta, \beta + \delta]$ are separated from $[\alpha, \beta]$ by a gap $\delta > 0$, then
\begin{equation}
\sin\Theta(V_A, V_B) \leq \frac{\|A - B\|_2}{\delta},
\label{eq:davis_kahan}
\end{equation}
where $\Theta$ denotes the matrix of principal angles between the two subspaces.
\end{theorem}

Applied to the leading eigenvector $q_1^0$ of $H_0$ and $q_1^d$ of $H_d$, with eigengap $\delta = \lambda_1^0 - \lambda_2^0$:
\begin{equation}
\sin\angle(q_1^0, q_1^d) \leq \frac{\|\Delta H\|_2}{\lambda_1^0 - \lambda_2^0}.
\label{eq:dk_applied}
\end{equation}

\textbf{This bound is of direct clinical significance.} When the healthy eigengap $\lambda_1^0 - \lambda_2^0$ is large, indicating that one axis of biomarker covariation dominates the healthy state, the leading disease eigenvector cannot rotate far from the healthy reference without a commensurately large perturbation. Conversely, when the healthy spectrum is flat (small eigengap), even modest disease perturbations can produce substantial eigenvector rotation, making the spectral basis unstable as a reference frame. This provides a criterion for cohort quality assessment, specifically where the eigengap of the reference Hamiltonian should be evaluated before spectral methods are applied.

\subsection{Spectral Fingerprint of Disease}

The combined information from eigenvalue shifts and eigenvector rotations defines the \emph{spectral fingerprint} of a disease state:
\begin{equation}
\mathcal{F}(H_d, H_0) = \left\{ (\delta\lambda_k,\, \angle(q_k^0, q_k^d)) \right\}_{k=1}^p.
\label{eq:fingerprint}
\end{equation}

This fingerprint is a $p$-dimensional characterization of how disease reorganizes biomarker covariance relative to the healthy state. \textbf{Disease conditions with distinct pathophysiology will produce distinct fingerprints}, even when their mean biomarker profiles overlap. Thus, spectral methods are particularly valuable in the clinical setting where mean-level biomarker differences are insufficient to discriminate disease subtypes.

\section{The Spectral Prognostic Score}
\label{sec:score}

\subsection{Derivation}

Given the spectral decomposition of the disease Hamiltonian $H_d = Q_d \Lambda_d Q_d^\top$ and the healthy isotropic variance $\sigma^2$, define the \emph{spectral prognostic score} for patient $i$ with biomarker vector $x_i \in \mathbb{R}^p$ as
\begin{equation}
\pi_i^{(k)} = x_i^\top q_k^d,
\label{eq:score}
\end{equation}
the projection of the patient's biomarker vector onto the $k$-th eigenmode of the disease Hamiltonian. The composite score across all $p$ modes is
\begin{equation}
\Pi_i = \frac{1}{2}\sum_{k=1}^p \left(\frac{1}{\sigma^2} - \frac{1}{\lambda_k^d}\right)(x_i^\top q_k^d)^2,
\label{eq:composite}
\end{equation}
a signed, eigenvalue-weighted sum of squared projections. Modes for which disease amplifies variance beyond the healthy baseline ($\lambda_k^d > \sigma^2$) receive positive weight, so that large projection onto those modes is evidence for disease. Modes for which disease suppresses variance ($\lambda_k^d < \sigma^2$) receive negative weight, so that large projection onto those modes is evidence against disease. The healthy variance $\sigma^2$ is estimable from the bulk of the reference eigenspectrum via the Marchenko-Pastur law, as described in Section~\ref{sec:formalism}. Figure [\ref{fig:disease_geometry}] shows what this score is doing geometrically. The raw biomarker space is uninformative, the eigenvector frame reorganizes it, and the eigenplane separates disease from stable. 

\begin{figure}[!t]
    \centering
    \includegraphics[width=\textwidth]{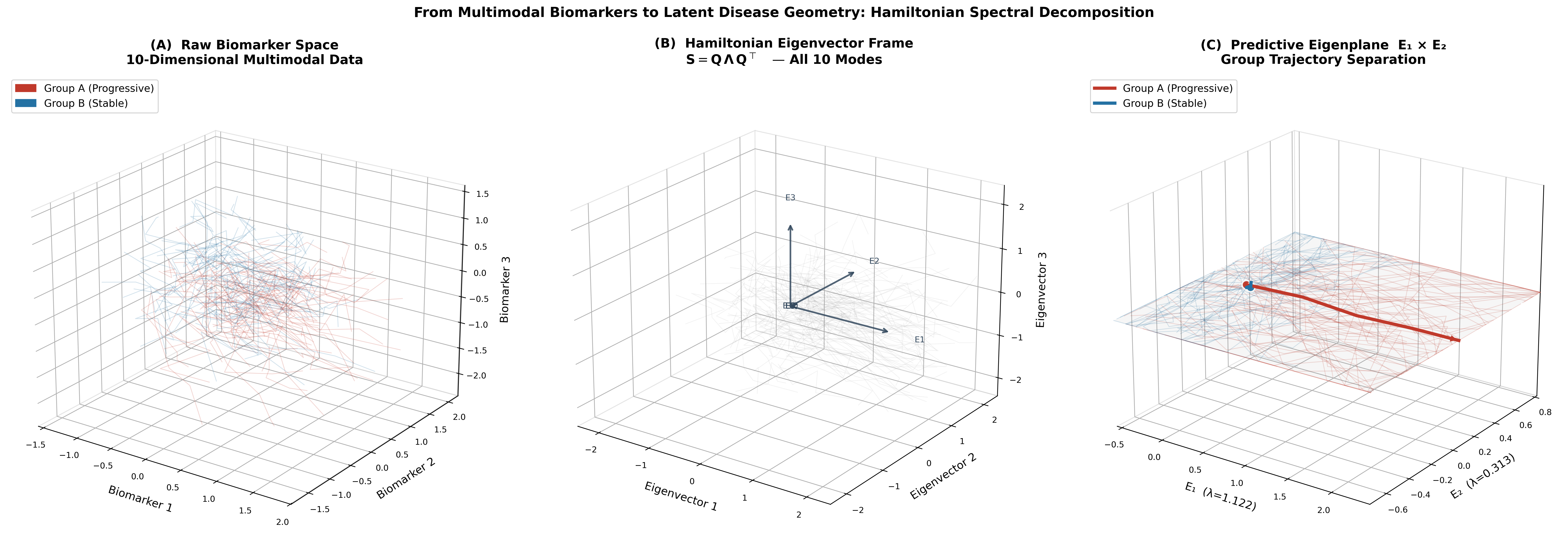}
    \caption{\textbf{From multimodal biomarkers to latent disease geometry via
    Hamiltonian spectral decomposition.}
    \textbf{(A) Raw biomarker space.} Ten-dimensional multimodal biomarker
    trajectories for Group A (progressive disease, red) and Group B (stable,
    blue), shown projected onto three representative biomarker axes. The two
    groups are indistinguishable in the original measurement space.
    \textbf{(B) Hamiltonian eigenvector frame.} The biomarker Hamiltonian
    $S = Q\Lambda Q^\top$ is decomposed into its 10 eigenmodes. The three
    dominant eigenvectors (E1, E2, E3) define a new coordinate frame that
    aligns with the principal axes of biomarker covariation across the cohort
    rather than with individual measurement dimensions.
    \textbf{(C) Predictive eigenplane $E_1 \times E_2$.} Patient trajectories
    projected onto the leading two eigenmodes ($\lambda_1 = 1.122$,
    $\lambda_2 = 0.313$) reveal complete separation of progressive from stable
    disease. Group A trajectories diverge monotonically along E1 while Group B
    remains clustered near the origin, demonstrating that the spectral
    prognostic score $\Pi_i$ recovers disease-discriminant structure that is
    invisible in the original biomarker space.}
    \label{fig:disease_geometry}
\end{figure}

\subsection{Optimality}

\begin{theorem}[Optimality of the Spectral Score]
\label{thm:optimality}
Let healthy patients $x \sim \mathcal{N}(0, H_0)$ and disease patients $x \sim \mathcal{N}(0, H_d)$ with $\mu_d = 0$ (pure covariance perturbation) and $H_0 = \sigma^2 I_p$ (isotropic healthy covariance), where $\sigma^2 > 0$. Then the Neyman-Pearson optimal test statistic for discriminating healthy from disease is, up to an additive constant,
\begin{equation}
\ell(x) = \frac{1}{2}\sum_{k=1}^p \left(\frac{1}{\sigma^2} - \frac{1}{\lambda_k^d}\right)(x^\top q_k^d)^2,
\end{equation}
which is identical to $\Pi_i$ defined in \eqref{eq:composite}. The score $\Pi_i$ is therefore the minimum-error-rate linear-quadratic classifier between the healthy and disease distributions under these assumptions.
\end{theorem}

\begin{proof}
The log-likelihood ratio for $x$ under disease versus healthy Gaussian is
\begin{align}
\ell(x) &= \log \frac{p_d(x)}{p_0(x)} \nonumber \\
&= -\frac{1}{2} x^\top H_d^{-1} x + \frac{1}{2} x^\top H_0^{-1} x + \frac{1}{2}\log\frac{\det H_0}{\det H_d} \nonumber \\
&= \frac{1}{2} x^\top (H_0^{-1} - H_d^{-1}) x + C,
\end{align}
where $C = \frac{1}{2}\log(\det H_0 / \det H_d)$ is a constant with respect to $x$. Setting $H_0 = \sigma^2 I_p$ gives $H_0^{-1} = \sigma^{-2} I_p$. Substituting the spectral decomposition $H_d^{-1} = Q_d \Lambda_d^{-1} Q_d^\top$ yields
\begin{equation}
\ell(x) = \frac{1}{2} \sum_{k=1}^p \left(\frac{1}{\sigma^2} - \frac{1}{\lambda_k^d}\right)(x^\top q_k^d)^2 + C,
\end{equation}
which equals $\Pi_i + C$. By the Neyman-Pearson lemma, the likelihood ratio test is the uniformly most powerful test at every significance level, so $\Pi_i$ is optimal. Modes for which $\lambda_k^d > \sigma^2$ contribute positively is evidence for disease. Modes for which $\lambda_k^d < \sigma^2$ contribute negatively is evidence against disease, more consistent with the interpretation that disease suppresses coordinated variation along those axes. This completes the proof.
\end{proof}

\begin{remark}
Theorem~\ref{thm:optimality} establishes the optimality of $\Pi_i$ under isotropic $H_0$. When $H_0$ is non-isotropic, the optimal statistic involves both $H_0^{-1}$ and $H_d^{-1}$, corresponding to the generalized eigenvalue problem $H_d v = \lambda H_0 v$. The solution eigenvectors are the \emph{discriminant spectral modes}, analogous to linear discriminant analysis but operating on full covariance structure rather than class means. The isotropic assumption is testable: the eigengap of the reference Hamiltonian $H_0$ and comparison to the Marchenko-Pastur bulk provide a direct diagnostic.
\end{remark}

\section{Eigenbasis Transfer Across Cohorts}
\label{sec:transfer}

A fundamental requirement for clinical deployment of spectral methods is \emph{eigenbasis transferability}, specifically the ability to apply an eigenbasis $Q_d$ estimated in a source cohort to compute prognostic scores in a target cohort drawn from a related but distinct population. We formalize the conditions under which this transfer is valid and quantify the expected degradation when those conditions are only approximately satisfied.

\subsection{Principal Angles and Subspace Distance}

Let $Q_s \in \mathbb{R}^{p \times r}$ and $Q_t \in \mathbb{R}^{p \times r}$ denote the top-$r$ eigenvectors of the source and target Hamiltonians, $H_s$ and $H_t$, respectively. The \emph{principal angles} $\theta_1 \leq \theta_2 \leq \cdots \leq \theta_r$ between the column spaces of $Q_s$ and $Q_t$ are defined via the singular value decomposition of $Q_s^\top Q_t = U \Sigma V^\top$, with $\cos\theta_k = \sigma_k$, where $\sigma_k$ are the singular values. Let $\Theta = \mathrm{diag}(\theta_1, \ldots, \theta_r)$ denote the diagonal matrix of principal angles. The subspace distance between source and target eigenspaces is
\begin{equation}
d(Q_s, Q_t) = \|\sin\Theta\|_F = \left(\sum_{k=1}^r \sin^2\theta_k\right)^{1/2},
\label{eq:subspace_distance}
\end{equation}
where $\|\cdot\|_F$ denotes the Frobenius norm.

\subsection{Transfer Bound}

\begin{proposition}[Prognostic Score Degradation Under Transfer]
Let $\pi^s_i = x_i^\top q_1^s$ and $\pi^t_i = x_i^\top q_1^t$ denote the leading spectral scores computed using the source and target eigenvectors, respectively. Then for any patient vector $x_i$ with $\|x_i\|_2 \leq M$,
\begin{equation}
|\pi^s_i - \pi^t_i| \leq M \cdot \|q_1^s - q_1^t\|_2 \leq M \cdot 2\sin\frac{\theta_1}{2},
\label{eq:transfer_bound}
\end{equation}
where $\theta_1$ is the leading principal angle between the source and target eigenspaces.
\end{proposition}

This bound has a direct practical interpretation. Before applying a source eigenbasis to a target cohort, one should compute the principal angle $\theta_1$ between the two leading eigenspaces. When $\theta_1$ is small, transfer is approximately lossless. When $\theta_1$ is large, the source basis does not represent the dominant axis of target variation, and transfer will produce systematically biased scores.

\subsection{Sufficient Conditions for Valid Transfer}

By the Davis-Kahan theorem \eqref{eq:dk_applied}, the principal angle between source, \textbf{$H_s$}, and target, \textbf{$H_t$}, leading eigenvectors satisfies
\begin{equation}
\sin\theta_1 \leq \frac{\|H_s - H_t\|_2}{\lambda_1^s - \lambda_2^s}.
\end{equation}

Transfer is therefore well-conditioned when (i) the Hamiltonians of the source and target populations are close in spectral norm, indicating shared underlying covariance structure, and (ii) the source eigengap $\lambda_1^s - \lambda_2^s$ is large, ensuring the leading eigenspace is well-separated and therefore stable under cohort-level variation.

Condition (i) is a statement about population homogeneity where the source and target cohorts should be drawn from populations with similar biomarker covariance structure which in practice requires shared procurement procedures, measurement platforms, and preprocessing pipelines. Condition (ii) is a statement about signal strength where the disease must produce a dominant and well-isolated spectral signature in the source cohort. Both conditions are discernible from the data prior to transfer.

\section{Unification of Existing Methods}
\label{sec:unification}

We now show that the principal methods in multivariate biomarker analysis and multiomics disease modeling are special cases of the spectral biomarker framework, distinguished by their choice of Hamiltonian construction, decomposition, and signal extraction.

\subsection{Principal Component Analysis}

PCA computes the spectral decomposition of the sample covariance matrix $H = X^\top X / n$ and represents patients in the coordinate system of the top-$k$ eigenvectors. This is identical to computing the spectral prognostic score \eqref{eq:score} for $k = 1, \ldots, r$ without a disease reference frame, $H_0$. PCA is therefore the special case of the spectral biomarker framework in which no healthy reference is available and signal extraction is performed by variance ordering alone without perturbative structure.

\subsection{Linear Discriminant Analysis}

LDA seeks the projection direction $w$ that maximizes the ratio of between-class variance to within-class variance:
\begin{equation}
w^* = \arg\max_w \frac{w^\top S_B w}{w^\top S_W w},
\end{equation}
where $S_B$ and $S_W$ are the between-class and within-class scatter matrices. The solution is the leading eigenvector of $S_W^{-1} S_B$, which is the leading solution to the generalized eigenvalue problem $S_B w = \lambda S_W w$. In the two-class setting with equal priors, letting $n_0$ and $n_d$ denote the number of healthy and disease patients, $\mu_0$ and $\mu_d$ their respective class mean vectors, and $n = n_0 + n_d$ the total cohort size, the scatter matrices are $S_B = \frac{n_0 n_d}{n}(\mu_d - \mu_0)(\mu_d - \mu_0)^\top$ and $S_W = \frac{n_0}{n} H_0 + \frac{n_d}{n} H_d$. LDA is therefore the special case of the spectral biomarker framework in which signal is extracted from the mean-shift component of the perturbation while the covariance perturbation $\Delta H$ is pooled into the within-class matrix and treated as noise.

\subsection{Canonical Correlation Analysis}

CCA seeks directions $u \in \mathbb{R}^{p_1}$ and $v \in \mathbb{R}^{p_2}$ that maximize the correlation between projections $Xu$ and $Yv$ of two data matrices $X \in \mathbb{R}^{n \times p_1}$ and $Y \in \mathbb{R}^{n \times p_2}$. The solution is the leading eigenvector of the cross-covariance operator $H_{XX}^{-1/2} H_{XY} H_{YY}^{-1/2}$, where $H_{XY} = X^\top Y / n$. CCA is the special case of the spectral biomarker framework in which the Hamiltonian is constructed from a cross-modal rather than within-modal covariance, and signal extraction targets maximally correlated modes across modalities rather than maximally variable modes within a single modality.

\subsection{Multiomic Factor Methods}

Methods such as MOFA \citep{Argelaguet2018} and JIVE \citep{Lock2013} decompose multi-block data matrices into shared and modality-specific factors. These factors are solutions to structured low-rank approximation problems that can be written as generalized eigenvalue problems on block-structured covariance matrices. They are special cases of the cross-modal Hamiltonian construction in which the block structure of $H$ is exploited to separate shared from modality-specific spectral modes. The spectral biomarker framework provides the parent object from which these block decompositions inherit their properties.

\subsection{Summary}

Table~\ref{tab:unification} summarizes the relationship between existing methods and the spectral biomarker framework.

\begin{table}[h]
\centering
\caption{Existing multivariate biomarker methods as special cases of the spectral framework.}
\label{tab:unification}
\small
\begin{tabular}{llll}
\toprule
\textbf{Method} & \textbf{Hamiltonian Construction} & \textbf{Decomposition} & \textbf{Signal Extraction} \\
\midrule
PCA & $H = X^\top X / n$ & Full eigendecomposition & Variance-ordered projection \\
LDA & $S_W^{-1} S_B$ & Generalized eigenproblem & Mean-shift discriminant \\
CCA & $H_{XX}^{-1/2} H_{XY} H_{YY}^{-1/2}$ & SVD of cross-covariance & Maximal cross-modal correlation \\
MOFA / JIVE & Block $H$ & Low-rank generalized & Shared vs.\ modality-specific modes \\
\textbf{This work} & $\Delta H = H_d - H_0$ & Perturbative eigenanalysis & Disease-discriminant spectral score \\
\bottomrule
\end{tabular}
\end{table}

\section{Discussion}
\label{sec:discussion}

\subsection{Relationship to Current Multiomics Trajectory Methods}

The methods reviewed in the current literature on multiomics disease trajectory analysis \citep{Nagin2018, NaginJones2024, ProustLima2022, Li2021, Maghsoudi2022, Ouwerkerk2023, Carrasco2024, Watanabe2023} are united by a focus on mean biomarker trajectories over time, while group-based trajectory models cluster patients by their mean biomarker evolution, and joint longitudinal models estimate individual mean trajectories with correlated random effects. Machine learning prediction models build scores from feature vectors of mean biomarker measurements at fixed time-points. What none of these methods characterizes is the evolving covariance structure of biomarker measurements as disease progresses, specifically the way in which the coordinated variation among biomarkers is reorganized by pathology.

The spectral biomarker framework addresses this gap directly. By operating on the covariance matrix rather than the mean vector, it captures exactly the information that mean-trajectory methods discard. Disease perturbation of $H_0$ reorganizes the eigenstructure of biomarker covariance. The spectral fingerprint \eqref{eq:fingerprint} quantifies this reorganization precisely. For diseases in which the primary pathological event is a disruption of regulatory coordination, as in neurodegeneration where loss of network synchrony precedes detectable mean-level biomarker changes, the spectral framework may detect disease signal earlier and more specifically than mean-trajectory methods.

\subsection{Thermodynamic Extension}

The analogy between the biomarker Hamiltonian and the quantum mechanical energy operator admits a natural thermodynamic extension. Let $\beta > 0$ denote an inverse temperature parameter. Define the partition function
\begin{equation}
Z(\beta) = \mathrm{Tr}[\exp(-\beta H)] = \sum_{k=1}^p e^{-\beta \lambda_k},
\label{eq:partition}
\end{equation}
where $\mathrm{Tr}[\cdot]$ denotes the matrix trace and the equality follows from the spectral decomposition of $H$. We define the free energy as 
\begin{equation}
F(\beta) = -\frac{1}{\beta} \log Z(\beta),
\label{eq:free_energy}
\end{equation} where the behavior of $F(\beta)$ encodes the full eigenvalue distribution of $H$. A phase transition in $F$ as a function of $\beta$, a non-analytic change in its derivatives, corresponds to a qualitative reorganization of the eigenspectrum, analogous to the phase transitions of statistical mechanics. In the finite-sample setting, this manifests as a gap in the sample eigenspectrum that separates disease signal from the Marchenko-Pastur background distribution. The critical inverse temperature $\beta^*$ at which this transition occurs is a scalar summary of the disease perturbation magnitude and will be formalized in subsequent work.

\subsection{Open Problems}

The framework presented here identifies several open problems for future development.

\textbf{Temporal eigenbasis dynamics.} The present framework treats $H$ as a static operator computed from a cross-sectional cohort. A natural extension defines a time-varying Hamiltonian $H(t)$ estimated from longitudinal data, with the disease trajectory characterized by the path $\{H(t) : t \geq 0\}$ in the space of positive semidefinite matrices. The rate of spectral change $dH/dt$ may itself constitute a clinically informative biomarker of disease acceleration.

\textbf{High-dimensional regime.} When $p \gg n$, the sample Hamiltonian $H = X^\top X / n$ is rank-deficient, with $p - n$ zero eigenvalues. Regularization via ridge ($H + \alpha I$), graphical lasso, or nuclear norm penalization is required. The choice of regularizer affects the spectral structure and the transfer bounds derived in Section~\ref{sec:transfer}. A systematic analysis of regularization effects on spectral disease characterization is an open problem.

\textbf{Missing and censored data.} Clinical multiomics data are frequently incomplete. The effect of missing biomarker measurements on the sample Hamiltonian and its eigenbasis is not characterized by the bounds derived here, which assume complete data. Extensions to incomplete observations, drawing on the literature of matrix completion and covariance estimation under missingness, are required before the framework can be applied to typical clinical datasets.

\textbf{Non-Gaussian biomarkers.} The optimality result in Theorem~\ref{thm:optimality} assumes Gaussian biomarker distributions. Many clinical biomarkers are skewed, zero-inflated, or bounded. The spectral framework is defined without reference to distributional assumptions on $X$, but the optimality of the spectral prognostic score in non-Gaussian settings requires separate analysis.

\subsection{Applications as Instantiations}

The formal results derived here provide the theoretical basis for a class of empirical spectral biomarker studies across disease domains. In neurological disease, Hamiltonian spectral decomposition of cerebrospinal fluid and clinical biomarker panels has been applied to Parkinson's disease staging and Alzheimer's disease conversion prediction, with spectral disruption indices consistent with the perturbation-theoretic predictions of Section~\ref{sec:perturbation}. In oncology, eigendecomposition of proteomic covariance matrices across cancer types reveals a universal malignancy coordinate that admits thermodynamic interpretation in the framework of Section~\ref{sec:discussion}. Critically, in each of these settings the corrected score $\Pi_i$, which assigns negative weight to modes suppressed by disease, outperforms naive variance-weighted projections when disease disrupts rather than amplifies the dominant covariance structure, providing empirical confirmation of Theorem~\ref{thm:optimality} across independent cohorts spanning neurodegeneration, genomic instability, and multimodal proteogenomics. These applications are reported separately; the present paper establishes the mathematical foundation they share.

\section{Conclusion}

We have proposed and formalized the spectral biomarker framework, centered on the biomarker Hamiltonian $H = X^\top X / n$ as the fundamental object of multivariate disease representation. The framework asserts that disease is a perturbation of the spectral structure of healthy biomarker covariance: eigenvalues shift and eigenvectors rotate in proportion to the severity and character of pathological disruption. We derived sharp bounds on these shifts using Weyl's inequality and the Davis-Kahan theorem, proved the optimality of the spectral prognostic score under Gaussian assumptions, established conditions for eigenbasis transfer across cohorts, and unified the principal methods of multivariate biomarker analysis as special cases of the framework. The resulting structure provides a principled mathematical foundation for spectral biomarker analysis and a vocabulary for relating existing and future methods within a common theoretical language.



\end{document}